%% file: camera-ready.tex
\documentclass[sigconf]{acmart}

\usepackage{amsmath,amssymb,amsfonts}
\usepackage{algorithmic}
\usepackage{graphicx}
\usepackage{textcomp}
\usepackage{xcolor}

\usepackage[utf8]{inputenc} 
\usepackage[T1]{fontenc}    
\usepackage{hyperref}       
\usepackage{url}            
\usepackage{booktabs}       
\usepackage{amsfonts}       
\usepackage{nicefrac}       
\usepackage{microtype}      
\usepackage{xcolor}         
\usepackage{color}
\usepackage{paralist}
\usepackage{amsmath}
\usepackage{graphicx}
\usepackage{subfigure}
\usepackage{array}
\usepackage{epstopdf}
\usepackage[ruled,linesnumbered]{algorithm2e}
\usepackage{multirow} 
\usepackage{tabularx}

\usepackage{float}
\usepackage{placeins}
\usepackage{caption}
\usepackage{afterpage}

\copyrightyear{2026}
\acmYear{2026}
\setcopyright{cc}
\setcctype{by}
\acmDOI{10.1145/3773966.3777937}

\acmConference[WSDM '26] {Proceedings of the Nineteenth ACM International Conference on Web Search and Data Mining}{February 22--26, 2026}{Boise, ID, USA.}
\acmBooktitle{Proceedings of the Nineteenth ACM International Conference on Web Search and Data Mining (WSDM '26), February 22--26, 2026, Boise, ID, USA}
\acmISBN{979-8-4007-2292-9/2026/02}

\def\oursys{Fast-DataShapley\xspace}

\begin{document}

\title{\oursys: Neural Modeling for Training Data Valuation}

\author{Haifeng Sun}
\affiliation{
  \institution{University of Science and Technology of China}
  \city{Hefei}
  \country{China}
}
\email{sun1998@mail.ustc.edu.cn}

\author{Yu Xiong}
\affiliation{
  \institution{Fuxi AI Lab, NetEase Inc.}
  \city{Hangzhou}
  \country{China}
}
\email{xiongyu1@corp.netease.com}

\author{Runze Wu}
\authornote{Corresponding author}
\affiliation{
  \institution{Fuxi AI Lab, NetEase Inc.}
  \city{Hangzhou}
  \country{China}
}
\email{wurunze1@corp.netease.com}

\author{Xinyu Cai}
\affiliation{
  \institution{Nanyang Technological University}
  \country{Singapore}
}
\email{xinyu.cai@ntu.edu.sg}

\author{Changjie Fan}

\affiliation{
  \institution{Fuxi AI Lab, NetEase Inc.}
  \city{Hangzhou}
  \country{China}
}
\email{fanchangjie@corp.netease.com}

\author{Lan Zhang}

\affiliation{
  \institution{University of Science and Technology of China}
  \city{Hefei}
  \country{China}
}
\email{zhanglan@ustc.edu.cn}

\author{Xiang-Yang Li}
\authornotemark[1]
\affiliation{
  \institution{University of Science and Technology of China}
  \city{Hefei}
  \country{China}
}
\email{xiangyangli@ustc.edu.cn}

\input{1_Abstract}



\maketitle

\input{2_Introduction}

\input{4_Preliminary}

\input{5_Framework}

\input{6_Experiment}

\input{3_Related_work}

\input{8_Conclusion}

\begin{acks}
The research is partially supported by  Quantum Science and Technology-National Science and Technology Major Project (QNMP) 2021ZD0302900 and China National Natural Science Foundation with No. 62132018, 62231015, "Pioneer” and “Leading Goose”R\&D Program of Zhejiang", 2023C01029, and 2023C01143.
\end{acks}

\appendix
\input{9_appendix}

\clearpage

\section{Ethical Considerations}

The proposed Fast-DataShapley framework aims to equitably evaluate the contributions of training data providers, which has implications for fairness and incentivization in data markets. However, potential negative societal impacts must be considered. Misuse of the framework by malicious actors could lead to biased data valuations, unfairly prioritizing certain data providers or underrepresenting others, potentially exacerbating existing inequalities in data markets. Additionally, the reliance on training data for valuation raises privacy concerns, as sensitive information could be inferred from the contributions assigned to specific datasets. Even when used as intended, inaccuracies in Shapley value approximations could lead to unfair compensation, undermining trust in AI systems. To mitigate these risks, we propose implementing robust auditing mechanisms to detect and correct biased valuations, ensuring transparency in the valuation process. Privacy-preserving techniques, such as differential privacy, should be integrated to protect sensitive data. Furthermore, clear guidelines and validation protocols for the explainer model’s outputs can help maintain fairness and trust, ensuring equitable treatment of all data providers.

\bibliographystyle{ACM-Reference-Format}
\bibliography{sample-base}

\end{document}

%% file: 1_Abstract.tex
\begin{abstract}
The value and copyright of training data are crucial in the artificial intelligence industry. 
Service platforms should protect data providers' legitimate rights and fairly reward them for their contributions. 
Shapley value, a potent tool for evaluating contributions, outperforms other methods in theory, but its computational overhead escalates exponentially with the number of data providers. 
Recent studies on Shapley values have proposed various approximation algorithms to address the computational complexity issues inherent in exact calculations.
However, they need to retrain for each test sample, leading to intolerable costs. 
We propose \oursys, a one-pass training framework that leverages the weighted least squares characterization of the Shapley value to train a reusable explainer model with real-time reasoning speed. 
Given new test samples, no retraining is required to calculate the Shapley values of the training data. 
Additionally, we propose three methods with theoretical guarantees to reduce training overhead from two aspects: the approximate calculation of the utility function and the reduction of the sample space complexity. 
We analyze time complexity to show the efficiency of our methods. 
The experimental evaluations on various image datasets demonstrate superior performance and efficiency compared to baselines. 
Specifically, the performance is improved to more than $2\times$, and the explainer's training speed can be increased by two orders of magnitude.

\end{abstract}

%% file: 2_Introduction.tex
\section{Introduction}

The rapid development of deep learning has led to the widespread adoption of artificial intelligence services in various fields, including science and business~\cite{wurman2022outracing,jumper2021highly}. 
Deep learning is built on training data, typically obtained from multiple providers of varying quality~\cite{reddi2021data}. Data is valuable because high-quality data can train high-performance models. Thus, data copyright and pricing are crucial~\cite{feng2021online,xue2022competitive}. Figure~\ref{scenario} illustrates an example of incentivizing data providers to provide high-quality data for artificial intelligence-generated content (AIGC) tasks. The revenue is distributed to the data providers for each generated sample. This not only assigns the reward the data providers deserve but also achieves copyright protection by data attribution. Therefore, ensuring equitable training data valuation is critical. 
 \begin{figure}[t]
  \centering
  \includegraphics[width=\linewidth]{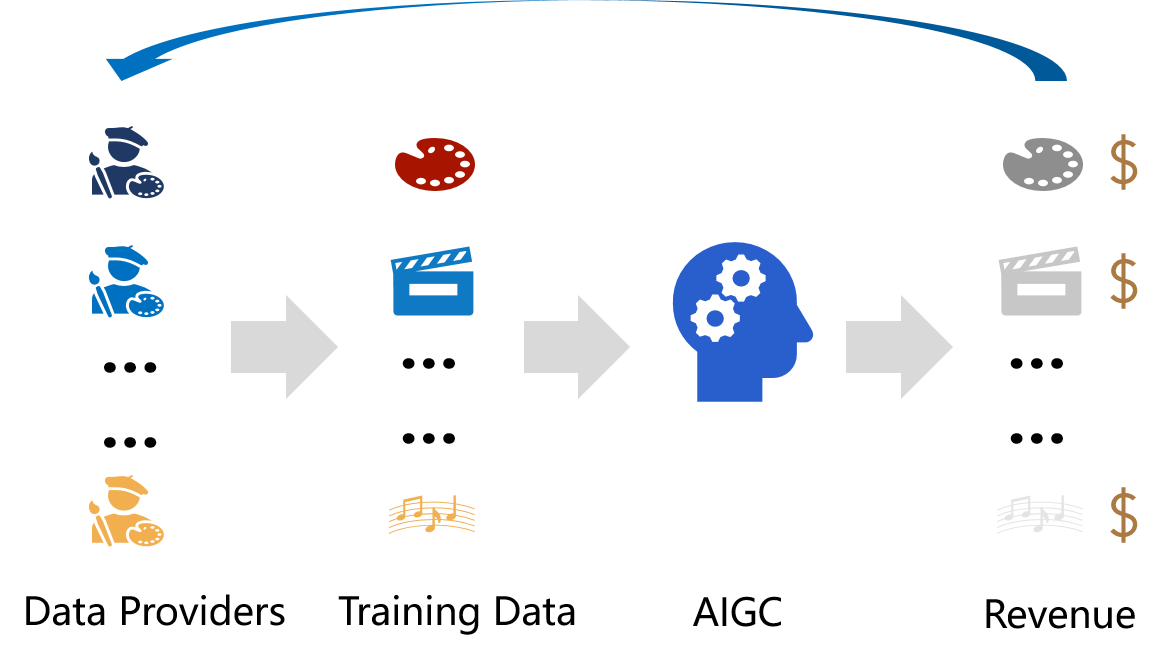}
  \caption{An example of motivating AIGC data providers.}
  \label{scenario}
\end{figure}

Currently, there exist two approaches~\cite{hammoudeh2022training} to evaluate the contributions of training data: \textbf{Gradient-Based Influence Estimators} (GBIEs) and \textbf{Retraining-Based Methods} (RBMs). 
GBIEs estimate the contributions by adjusting the gradients of training and test samples during or after the training process~\cite{koh2017understanding,guo2020fastif,schioppa2022scaling,yeh2018representer,sui2021representer,pruthi2020estimating,yeh2022first,hara2019data,chen2021hydra, wang2024helpful}. These methods are efficient, but makes many assumptions about the learning environment~\cite{hammoudeh2022training}, such as convexity.
RBMs~\cite{ling1984residuals,black2021leave,jia2021scalability,sharchilev2018finding,wojnowicz2016influence,feldman2020neural,kandpal2022deduplicating,jiang2020characterizing,jia2019efficient,ghorbani2019data}
provide this simplicity without any assumptions by retraining the model with different subsets of the training data. For example, the leave-one-out method~\cite{ling1984residuals} involves training a new model by removing a single training sample, and the performance difference between the new and original models determines the contribution of the removed sample.
Among existing RBMs, the Shapley value methods consider the joint influence relationship between data points and have excellent equitability due to their strong theoretical foundations. They view supervised machine learning as a cooperative game to train a target model, with each training data as a player. It is worth noting, however, that the computational complexity of the Shapley value methods escalates exponentially with the number of players, which is significantly higher than other methods.
Thus, many Shapley-based methods~\cite{ghorbani2019data,jia2019towards,schoch2022cs} have explored approximate approaches to calculating Shapley values for data valuation. 
However, these approximate methods evaluate the contributions of the training data to the entire test set. However, we focus on the contributions to each test sample. When given one new test sample, Shapley values need to be recalculated, which involves retraining the target model numerous times, and the overhead increases linearly with the number of prediction samples.

To tackle the challenges mentioned above, we propose \oursys, which trains a reusable explainer to evaluate the training data's contributions to any new test samples without any target model retraining. 
We first define the contribution of training data to a single test sample and design a new utility function to evaluate such contribution.
Inspired by the weighted least squares characterization of the Shapley values~\cite{charnes1988extremal,jethani2021fastshap}, we use an objective function minimization to train the explainer model without knowing the ground truth Shapley values. The model can infer the Shapley values of the training data for any new test samples in real time. 
To further reduce the one-pass training cost, we also propose three techniques: Approximate Fast-DataShapley ($AFDS$), Grouping Fast-DataShapley ($GFDS$), and Grouping Fast-DataShapley Plus ($GFDS^+$), which trade off the training efficiency and accuracy of the explainer model. $AFDS$ approximates the true utility value based on the loss information in the first few training epochs. 
$GFDS$ and $GFDS^+$ leverage the ideas of Owen value and Shapley value symmetry, respectively.
These two approaches group the training data by designing different grouping functions and then treat each group as a separate coalition player when calculating the Shapley values. As a result, the time complexity of the calculation is significantly reduced.

The contributions of our work are summarized as follows: 

$\bullet$  To the best of our knowledge, we are the first to apply a trained explainer for equitably evaluating the contribution of each training data. The explainer can predict the Shapley values for a new test sample in real-time, without retraining.

$\bullet$ We propose three methods to reduce the training overhead of the explainer model. $AFDS$ uses part of the training information during retraining to estimate the true value of the utility function. $GFDS$ and $GFDS^+$ employ the idea of grouping to reduce the number of combinations of calculating Shapley values. We also do some theoretical error analysis to support our methods.

$\bullet$ We propose a new evaluation metric named value loss to measure the contribution of the training data for an individual test sample.
Extensive experimental results on various datasets demonstrate the performance and efficiency of our methods. Specifically, the average performance is 2.5 to 13.1 times that of the baseline. Furthermore, the training speed of the explainer can be increased by a factor of 100+.

%% file: 4_Preliminary.tex
\section{Preliminaries}

In this work, we focus on the data valuation problem in supervised classification tasks. Let $x \in \mathbb{R}^d$ represent a d-dimensional feature vector that follows the distribution $X$, denoted as $x \sim X$. The ground truth label of a data point is denoted by $y \in Y = \{1,\dots,m\}$. The training data, denoted as $D$, can be expressed as $\{(x_1,y_1),\dots,(x_n,y_n)\}$, where $x_i=(x_{i}^1,\dots,x_{i}^d)$ and $y_i\in Y$. We use $s\in \{0,1\}^n$ to represent the subset of the indices $I = \{1,\dots,n\}$. The subset of the training data, denoted as $D_s$, is defined as $D_s = \{(x_i,y_i)|i:s_i=1\}$.

\subsection{Prediction-specific Training Data Valuation Framework}
\begin{figure}[htbp]
  \centering
  \includegraphics[width=\linewidth]{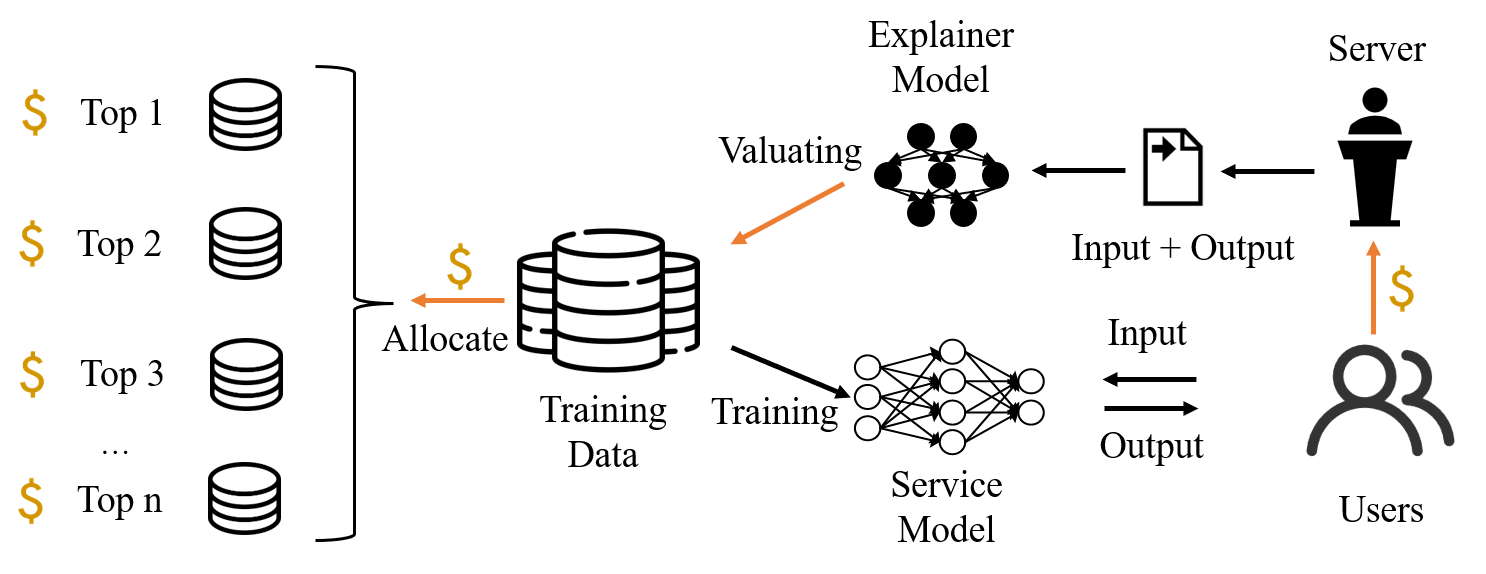}
  \caption{The overview of the framework of return assignment.}
  \label{framework}
   \vspace{-0.2in}
\end{figure}
As shown in Figure \ref{framework}, it is assumed that the training data $D$ is provided by $n$ data providers respectively. The server utilizes the training data to train a service model (target model) $f_D: X\to \mathbb{R}^m$ that users will pay for using the trained service model. Let $x$ be a test sample to be predicted, and $y = \arg\max_l (f_D(x)_l)$ is the output. To achieve data copyright protection and incentivize data providers to provide high-quality data, the server needs to assign appropriate money to each provider based on their contribution to the current prediction task.  Let $c$ represent the total money distributed to the data providers. This framework employs a trained explainer model $\phi_\theta: X\times Y\to \mathbb{R}^n$ to infer the Shapley values of each training data for predicting the {test} sample $x$ as the label $y$. The money allocated to the i-th data provider is denoted as $c_i$ and is calculated as follows:
$c_i = \frac{\phi_\theta(x,y)_i}{\sum_{j=1}^n\phi_\theta(x,y)_j}c$. 
So how do we define the contribution of training data? Existing work defines the contribution of training data by the accuracy of the model improved on the entire test set. Different from previous definitions, we focus on the prediction-specific scenario where the model predicts a single sample. We define the contribution of different training data by evaluating the improvement of the model's prediction confidence of that sample.

\subsection{DataShapley}
DataShapley~\cite{ghorbani2019data} first utilizes Shapley values to solve the training data value assignment problem.
The Shapley value (SV)~\cite{shapley1953value} is commonly used in game theory to identify the contributions of players collaborating in a coalition.
Given a value function $v: \{0,1\}^n\to \mathbb{R}$, the SV of the training data $(x_i, y_i)$ is defined as $\phi_i(v)$, which is the average marginal contribution of $(x_i, y_i)$ to all possible subsets of $D$: $\phi_i(v) = \frac{1}{n}\sum_{s_i\neq 1}{\tbinom{n-1}{\vec1^T s}}^{-1}(v(s+e_i)-v(s))$.
The quantity $v(s+e_i)-v(s)$ represents the marginal contribution of the training data $(x_i, y_i)$ to the subset $s$.
The SV satisfies the following four properties:
Effectiveness, Symmetry, Redundancy and  Additivity.
However, calculating the SVs of the training data requires an exponentially large number of calculations relative to the size of the training data. Thus, DataShapley proposes Truncated Monte Carlo Shapley to tackle this problem by extending Monte Carlo approximation methods. 
Initially, DataShapley samples random permutations of the training data, incrementally adding data from first to last. For each data point, it computes the marginal contribution gained after its inclusion. This process can be repeated, with the mean of all marginal contributions for a given data point used to refine the approximation. As the volume of training data increases, the marginal contribution of subsequent data diminishes~\cite{mahajan2018exploring,beleites2013sample}. Hence, DataShapley truncates the calculation and approximates the remaining data's marginal contributions as zero when the marginal contribution becomes negligible.

\subsection{FastSHAP}
In the field of feature importance analysis, KernelSHAP and FastSHAP are proposed to accelerate the computation of SVs. 

\textbf{KernelSHAP}~\cite{lundberg2017unified} trades off the accuracy and computation cost by sampling from the Shapley kernel. Given a test sample $x\in \mathbb{R}^d$, a label $y\in Y$ and a value function $u_{x,y}: \{0,1\}^d\to \mathbb{R}$, the SVs $\phi_{kernel}\in \mathbb{R}^d$ of the sample features are the solution of the Constrained Weighted Least Squares (CWLS) optimization  problem~(\ref{4.2}):
\begin{equation}
\begin{aligned}
\phi_{kernel}(u_{x,y}) &= \arg\min_{\phi_{x,y}} \ \mathbb{E}_{P(\bar{s})}[(u_{x,y}(\bar{s})-u_{x,y}(\vec 0)-\bar{s}^T\phi_{x,y})^2]\\
&s.t.\  \vec1^T\phi_{x,y} = u_{x,y}(\vec 1)-u_{x,y}(\vec 0),\\
&where \  P(\bar{s})\propto\frac{d-1}{\tbinom{d}{ \vec1^T \bar{s}}\cdot  \vec1^T \bar{s} \cdot (d- \vec1^T \bar{s}) }.
\end{aligned}
 \label{4.2}
\end{equation}
 $\bar{s}\in \{0,1\}^d$ is the subset of the features' indices and meets that $0<\vec1^T \bar{s}<d$~\cite{charnes1988extremal}. The distribution $P(\bar{s})$ is the Shapley kernel. $ \vec1^T\phi_{x,y} = u_{x,y}(\vec 1)-u_{x,y}(\vec 0)$ is the efficiency constraint.

\textbf{FastSHAP}~\cite{jethani2021fastshap} uses a learned function $\phi_{fast}(x,y,\theta):(X, Y)\to \mathbb{R}^d$ to generate SV explanations for sample features  by minimizing the following loss function:
$ Loss(\theta) = \mathbb{E}_{p(x)} \mathbb{E}_{Unif(y)} \mathbb{E}_{p(\bar{s})}[(u_{x,y}(\bar{s})-u_{x,y}(\vec 0)-\bar{s}^T\phi_{fast}(x,y,\theta))^2]$,
where $p(x)$ represents the distribution of sample $x$, $Unif(y)$ represents that $y$ is uniformly selected.
FastSHAP has two main advantages compared to existing methods: (1) it avoids solving a separate optimization problem for a new sample that needs to be interpreted, and (2) it can compute SVs in real time.
Inspired by FastSHAP, we train a reusable explainer model to calculate the SVs of a model's training data for test samples instead of repeatedly retraining the target model.

%% file: 5_Framework.tex
\section{Methods}
In this section, we first introduce \oursys to train the explainer model. Then, we propose three methods to reduce the training overhead. The relationships between these methods are illustrated in Figure \ref{method}. 
\oursys is based on CWLS optimization of SVs. Firstly, the value function is calculated through Shapley kernel sampling, then the SVs are updated through efficiency constraints, and finally, the explainer model is iterated by optimizing the loss. $AFDS$ reduces the training overhead by approximating the value function. It uses the average utility value from previous rounds to estimate the true value.
$GFDS$ reduces computational complexity by first grouping the training data and then treating each other group separately as a whole. 
$GFDS^+$ treats each group separately as a whole with no distinction between in-group.
Next, we will describe each method in detail.
\begin{figure*}[ht]
  \centering
  \includegraphics[width=\linewidth]{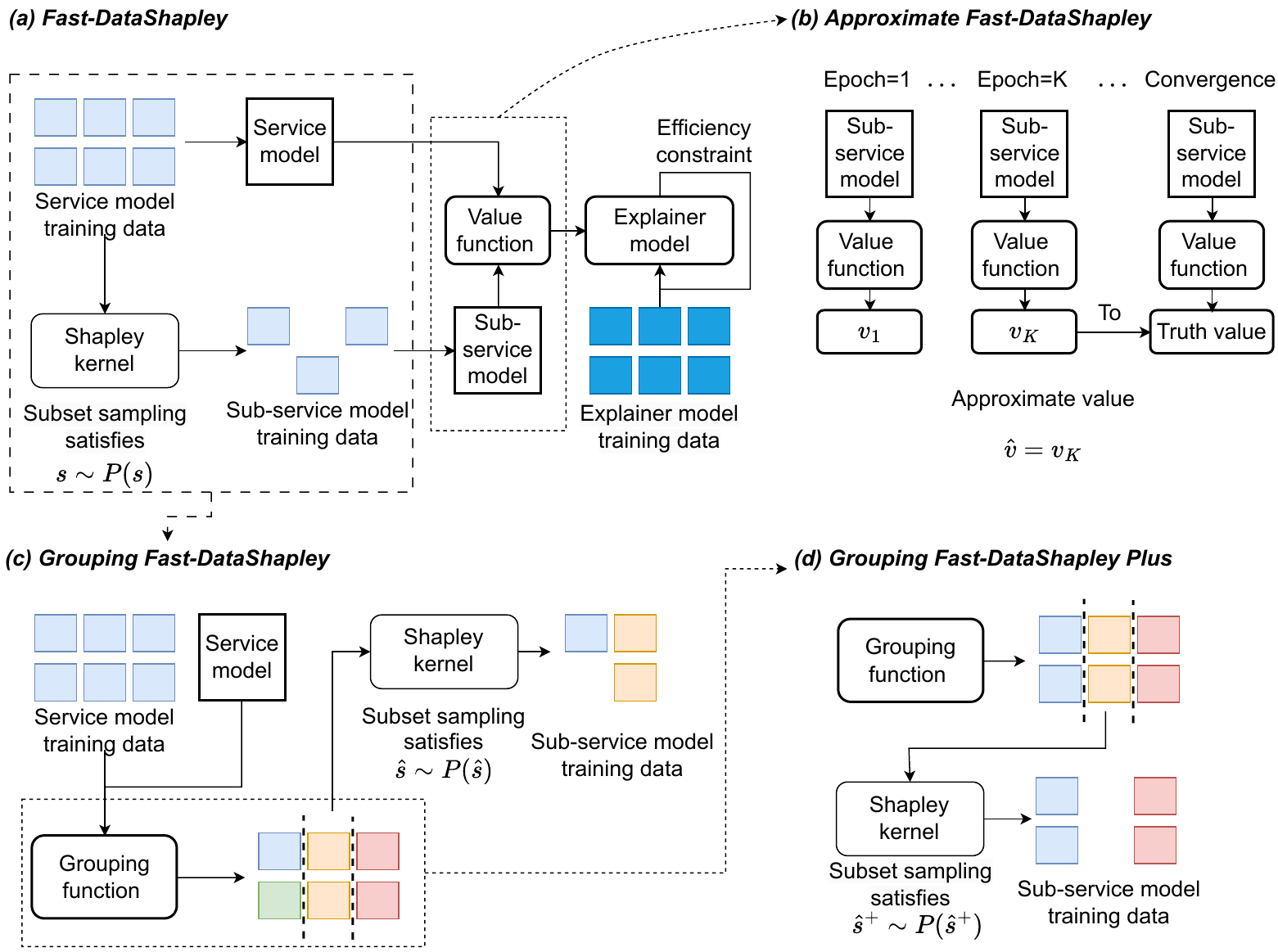}
  \caption{(a) The training of the explainer model mainly includes two core parts: the calculation of the value function and the other is updating the model through efficiency constraint. (b) Utilizing information from the first few epochs of sub-service model training to estimate the true value of utility. (c) A grouping function groups the training data, and each data of the same group and the overall data of other groups are regarded as players of the grand coalition to calculate the SV. (d) After grouping, each group is treated as a single player in the grand coalition.}
  \label{method}
\end{figure*}

\subsection{ \oursys }

\oursys leverages the idea of FastSHAP to train the explainer model. The explainer model takes the current predicted sample $x$ and the predicted label $y$ of the service model $f_D$ as input and outputs the SVs $\phi_\theta(x,y)$ of the training data $D$. Let $f_{D_s}: X\to \mathbb{R}^m$ denote the sub-service model trained using the sub-training data $D_s = \{(x_i,y_i)|i:s_i=1\}$. Figure \ref{method}
 (a) gives the \oursys training process. To evaluate the contribution of different training data subsets to the service model's prediction of the sample $x$ as $y$, \oursys designs the following value function $v_{x,y}(s)$ to train the explainer model by confidence in model predictions: 
\begin{equation}
\begin{aligned}
v_{x,y}(s) =\sigma(f_{D_s}(x))_y,
\end{aligned}
 \label{5.32}
\end{equation}
where $\sigma: \mathbb{R}^m\to (0,1)^m$ is a standard softmax function and is defined by the formula $\sigma(z)_i = \frac{e^{z_i}}{\sum_{j=1}^m e^{z_j}}$, for $i=1,\dots,m$ and $z=(z_1,\dots,z_m)\in \mathbb{R}^m$. 
The loss function of the explainer model $\phi_\theta: X\times Y\to \mathbb{R}^n$ is defined as follows:
\begin{equation}
\begin{aligned}
 L(\theta) &= \mathbb{E}_{x\sim X,y\sim U(Y),s\sim P(s)}[(v_{x,y}(s)-v_{x,y}(\vec 0)\\
&-s^T\phi_\theta(x,y))^2].\\
\end{aligned}
 \label{5.31}
\end{equation}
Here, $y$ is sampled uniformly from the label set $Y$. $s$ is sampled from the Shapley kernel $P(s)$, where $P(s)=\frac{n-1}{\tbinom{n}{ \vec1^T s}\cdot  \vec1^T s \cdot (n- \vec1^T s) }$.
To satisfy the efficiency constraint of SVs, \oursys  adjusts the SV predictions by the additive efficient normalization~\cite{ruiz1998family}, i.e., 
$\hat{\phi} \leftarrow \hat{\phi} + n^{-1}(v_{x,y}(\vec 1) - v_{x,y}(\vec 0) - {\vec 1}^T\hat{\phi})$.
     
        

 
\subsection{ Approximate \oursys ($AFDS$)}
However, computing the value function $v_{x,y}(s)$ requires retraining many sub-service models to converge, which can bring a lot of training overhead. Then, we propose a new approximation algorithm, $AFDS$, to alleviate these limitations.
Inspired by G-Shapley~\cite{ghorbani2019data}, which uses partial gradient information to approximate the value function, $AFDS$ uses information from the first few training epochs of the sub-service models to estimate the true value of the utility instead of training the sub-service models to converge, thereby greatly reducing the cost of training.
As shown in Figure \ref{method} (b),
$AFDS$ designs the value function: 
\begin{equation}
\begin{aligned}
v_{x,y,K}(s) = \sigma(f^K_{D_s}(x))_y.
\end{aligned}
 \label{5.1}
\end{equation}
The hyperparameter $K$ indicates the number of epochs used to train the sub-service models. A sub-service model trained at the i-th epoch can be denoted as $f_{D_s}^i: X\to \mathbb{R}^m$. 
This value function estimates the true value of the utility by taking the $K-th$ epoch. 
The explainer model trained by $AFDS$ can be represented as $\phi_{\theta,K}: X\times Y\to \mathbb{R}^n$. The loss function of this explainer model, denoted as $L(\theta,K)$, can be expressed as follows: 
\begin{equation}
\begin{aligned}
 L(\theta,K) &= \mathbb{E}_{x\sim X,y\sim U(Y),s\sim P(s)}[(v_{x,y,K}(s)-v_{x,y,K}(\vec 0)\\
&-s^T\phi_{\theta,K}(x,y))^2]\\
\end{aligned}
 \label{5.3}
\end{equation}

To improve the accuracy of the value function $v_{x,y, K}$ during the initial $K$ epochs of training the sub-service models, we speed up the convergence of the sub-service models by adjusting the learning rate. Specifically, the learning rate for training the service model $f_D$ is $lr$, then the learning rate for training the sub-service model $f^i_{D_s}$ is $\beta lr$, where $\beta$ is a hyperparameter. 
Additionally, to satisfy the efficiency constraint, $AFDS$ adjusts the SVs predicted by the explainer model using the following formula:
$\hat{\phi} \leftarrow \hat{\phi} + n^{-1}(v_{x,y,K}(\vec 1)-v_{x,y,K}(\vec 0)- {\vec 1}^T\hat{\phi})$.

\subsection{ Grouping \oursys ($GFDS$)}
In addition to approximating the value function, we can also reduce training overhead from the perspective of group sampling. 
We introduce $GFDS$, which leverages the Owen value's idea \cite{gimenez2019owen,casajus2009shapley}, which is close to the SV, significantly reducing the training overhead by reducing the number of players in the coalition.
\begin{algorithm}
	\renewcommand{\algorithmicrequire}{\textbf{Require:}}
	\caption{Grouping \oursys Training}
	\label{alg:3}
    \KwIn{Service model training data $D$, service model $f_D$, grouping function $\Psi_N$, value function $v_{x,y,K}$, explainer learning rate  $\alpha$, service model learning rate $lr$, hyperparameters $K,N,\beta$, sub-service model learning rate $\beta lr$}
    \KwOut{Explainer model: $\hat{\phi}_{\theta,K}(x,y)$}
     
	\begin{algorithmic}[1]
		\STATE Initialize $\hat{\phi}_{\theta,K}(x,y)$
		\WHILE{not converged}
            \STATE sample $x\sim p(x), y\sim U(Y), \hat{s}\sim P(\hat{s})$ 
            \STATE $\{D_1,\dots,D_N\} = \Psi_N(D,f_D)$
            \FOR{  $j\in \{1,\cdots,N\}$}
            
            \STATE $\hat{D} = \{D_i|i\neq j\}\bigcup D_j$
            \FOR{ epoch $i\in \{1,\cdots,K\}$}
                \STATE $f^K_{\hat{D}_{\hat{s}}}$ $\leftarrow$ train $f^i_{\hat{D}_{\hat{s}}}$ at epoch $K$
                \STATE  $v_{x,y,K}(\hat{s})\leftarrow \sigma(f^K_{\hat{D}_{\hat{s}}}(x))_y$ 
            \ENDFOR
            
            \STATE predict $\hat{\phi} \leftarrow \hat{\phi}_{\theta,K}(x,y)$ 
            \STATE $\hat{\phi} \leftarrow \hat{\phi} + n^{-1}(v_{x,y,K}(\vec 1) - v_{x,y,K}(\vec 0)- {\vec 1}^T\hat{\phi})$
            \STATE $s\leftarrow \hat{s}$
            \STATE loss $L\leftarrow (v_{x,y,K}(\hat{s}) - v_{x,y,K}(\vec 0)- s^T \hat{\phi})^2$
            
            \STATE update $\theta \leftarrow \theta - \alpha \nabla L$ 
            \ENDFOR
        \ENDWHILE
        
	\end{algorithmic}  
\end{algorithm} 
As shown in Figure \ref{method} (c), $GFDS$ uses a model-dependent grouping function $g$, such as clustering according to data labels or feature embeddings, to group the training data $D$ into $N$ mutually disjoint subsets, i.e., $g(D,f_D)=\{D_1,\dots,D_N\}$, s.t., $D_i\cap D_j=\emptyset$ and $D=\bigcup_{i=1}^ND_i$. 
Suppose $(x_i, y_i)$ belongs to $D_j$. When calculating the SV of $(x_i, y_i)$, all training data inside $D_j$ are individually regarded as the players of the grand coalition. Each subset $D_i$, where $i\neq j$, is regarded as one player. So, the number of players in the coalition is reduced from $n$ to $|D_j|+N-1$. Suppose the order of $(x_i, y_i)$ as a player in the grand coalition is $t$, its SV is denoted as $\varphi_i$ and can be calculated as follows:
$\varphi_i = \frac{1}{|D_j|+N-1}\sum_{\hat{s}_t\neq 1}{\tbinom{|D_j|+N-1-1}{\vec1^T \hat{s}}}^{-1}(v_{x,y}(\hat{s}+e_t)-v_{x,y}(\hat{s}))$,
where $\hat{s}\in \{0,1\}^{|D_j|+N-1}$.
The time complexity of SV calculation of \oursys is $\mathcal{O}(2^n)$, but the time complexity of $GFDS$ is $\mathcal{O}(2^{{|D_j|+N-1}})$. Assuming that the training data $D$ can be divided into $N$ equal parts, the time complexity is $\mathcal{O}(2^{{|\frac{n}{N}|+N-1}})$, when $N=\sqrt n$, the time complexity is the lowest. When $N=1$ or $N=n$, $\varphi$ is equivalent to SVs. $GFDS$  reduces the time complexity by selecting an appropriate $N$, reducing the explainer model's training overhead. It balances the accuracy of calculating the SVs and the training overhead.
In Algorithm \ref{alg:3}, the explainer model trained by $GFDS$ can be represented as $\hat{\phi}_{\theta,K}: X\times Y\to \mathbb{R}^n$. the grouping function defined by $GFDS$ is denoted as $g(D,f_D)=\Psi_N(D,f_D)$, where $\Psi_N$ is a clustering method that divides the training data with similar outputs of the service model into $N$ categories, i.e., $\Psi_N: \{ f_D(x_i)|i=1,\dots,n\}\to\{D_i|i=1,\dots, N\}$. 
In line 13 of the algorithm, notice that $GFDS$ uses $\hat{D}_{\hat{s}}=D_s$ to transform $\hat{s}$ into $s$ to align $s$ with $\hat{\phi}$ when updating the loss.
     
        

\subsection{ Grouping \oursys Plus (\texorpdfstring{$GFDS^+$}.)}
To calculate the contributions of the training data at a more coarse-grained level, we introduce $GFDS^+$,  
which assumes data in the same subset have the same contribution by the symmetry of the SVs. Each subset can be regarded as one player in the coalition. As shown in Figure \ref{method} (d), $GFDS^+$ can divide the training data into $N$ parts, it can reduce the time complexity from $\mathcal{O}(2^n)$ to $\mathcal{O}(2^N)$. Essentially, both $GFDS$ and $GFDS^+$ sacrifice the accuracy of calculating the SV in exchange for 
training efficiency.
Suppose $(x_{i},y_i)\in D_j$, its SV is denoted as $\varphi^+_i$ and can be calculated as follows:
$\varphi_i^+ = \frac{1}{|D_j|}\cdot\frac{1}{N}\sum_{\hat{s}^+_j\neq 1}{\tbinom{N-1}{\vec1^T \hat{s}^+}}^{-1}(v_{x,y}(\hat{s}^++e_j)-v_{x,y}(\hat{s}^+))$,
where $\hat{s}^+\in \{0,1\}^{N}$.
The formula is equivalent to first calculating the SVs that should be assigned to each subset $D_i$ and then evenly dividing the SVs to each subset data. 
To make it, we provide two solutions. The first is to change the explainer model's output dimension to $N$ and then distribute the predicted values evenly within the corresponding subset. The second is to design the following loss function to train the explainer model $\hat{\phi}^+_{\theta, K}: X\times Y\to \mathbb{R}^n$:
\begin{equation}
\begin{aligned}
&L_{GFDS^+}(\theta,K) = \mathbb{E}_{x\sim X,y\sim U(Y),\hat{s}^+\sim P(\hat{s}^+)}[(v_{x,y,K}(\hat{s}^+)-\\
&v_{x,y,K}(\vec 0)-s^T\hat{\phi}^+_{\theta,K})^2+\gamma\sum_{i=1}^N||\hat{\phi}^+_{\theta,K,D_i}-
\hat{\phi}^+_{\theta,K,D_i}/|D_i| ||].\\
\end{aligned}
 \label{5.6}
\end{equation}
Here, $\hat{\phi}^+_{\theta,K,D_i} \in \mathbb{R}^{|D_i|}$ represents the SVs predicted to the subset $D_i$ and $\gamma$ is a hyperparameter.
The second term means to reduce the distance between each data's SV in the same subset.  

\begin{table*}[t]
\newcommand{\tabincell}[2]{\begin{tabular}{@{}#1@{}}#2\end{tabular}}
\renewcommand\arraystretch{1.1}
\centering
\caption{$H_\eta$ of different methods on MNIST.}
\begin{tabular}{|p{3.1cm}<{\centering}|p{2.2cm}<{\centering}|p{2.2cm}<{\centering}|p{2.2cm}<{\centering}|p{2.2cm}<{\centering}|}
\hline
\bf{Methods}& \bf{$H_{5\%}\uparrow$} & \bf{$H_{10\%}\uparrow$} & \bf{$H_{15\%}\uparrow$ }& \bf{$H_{20\%}\uparrow$} \\
\hline
LOO   & 0.4560$\pm$0.1922 & 0.5058$\pm$0.2116 & 0.6983$\pm$0.3484 &  0.7456$\pm$0.3444 \\ 
\hline
DataShapley  & 0.4612$\pm$0.1679 & 0.5022$\pm$0.1539 & 0.6169$\pm$0.1992 & 0.7057$\pm$0.2102 \\ 
\hline
Group testing Shapley  & 0.4153$\pm$0.2405 & 0.5771$\pm$0.3256 & 0.7406$\pm$0.2747 & 0.8079$\pm$0.2681 \\ 
\hline
CS-Shapley  & 0.4622$\pm$0.3163 & 0.8972$\pm$0.6097 & 0.9363$\pm$0.5131 & 0.9830$\pm$0.5066 \\ 
\hline
FreeShap & 0.9218$\pm$0.1623 & 1.0582$\pm$0.1876 & 1.2546$\pm$0.2145 & 1.4576$\pm$0.2559 \\
\hline
\oursys (Ours)  & \bf{1.1985}$\pm$ 0.4043 & 4.0102$\pm$0.8477 & \bf{7.0062}$\pm$1.3877 & \bf{7.8676}$\pm$1.5516 \\ 
\hline
$AFDS$ (Ours)  & 0.6022$\pm$0.3048 & 0.9221$\pm$0.5201 & 1.2190$\pm$0.5477 & 1.3096$\pm$ 0.5291 \\ 
\hline
$GFDS$ (Ours)  & 0.8551$\pm$0.1594 &  1.0644$\pm$0.2613 & 1.4101$\pm$0.7337 & 1.9470$\pm$0.7394 \\ 
\hline
$GFDS^+$ (Ours)  & 1.0106$\pm$0.3636 & \bf{5.1576}$\pm$1.3814 & 6.1327$\pm$1.9130 & 6.5753$\pm$1.1251 \\ 
\hline
\end{tabular}
\label{table1}
\end{table*}

\begin{table*}[t]
\newcommand{\tabincell}[2]{\begin{tabular}{@{}#1@{}}#2\end{tabular}}
\renewcommand\arraystretch{1.1}
\centering
\caption{$H_\eta$ of different methods on CIFAR-10.}
\begin{tabular}{|p{3.1cm}<{\centering}|p{2.2cm}<{\centering}|p{2.2cm}<{\centering}|p{2.2cm}<{\centering}|p{2.2cm}<{\centering}|}
\hline
\bf{Methods}& \bf{$H_{5\%}\uparrow$} & \bf{$H_{10\%}\uparrow$} & \bf{$H_{15\%}\uparrow$ }& \bf{$H_{20\%}\uparrow$} \\
\hline
LOO             & 1.0244$\pm$0.6185  & 1.0621$\pm$0.6095  & 0.9483$\pm$0.4203  & 0.9679$\pm$0.3478\\ 
\hline
DataShapley     &  0.8411$\pm$0.4231 &  1.1437$\pm$0.4781 &  1.5334$\pm$1.0600 & 1.7496$\pm$1.6123\\ 
\hline
Group testing Shapley    &  0.8276$\pm$0.7340 &  1.6508$\pm$1.5221 &  2.0636$\pm$1.5102 & 3.0671$\pm$2.7041\\ 
\hline
CS-Shapley  & 0.5096$\pm$ 0.4118 & 0.5781$\pm$0.5533 & 0.7761$\pm$0.5005 & 1.1954$\pm$0.9665 \\ 
\hline
FreeShap & 1.6546$\pm$0.3745 & 2.2474$\pm$0.5987 & 2.9546$\pm$0.3464 & 3.3456$\pm$0.4832\\
\hline
\oursys (Ours)  & \bf{3.3311$\pm$1.3348} & 4.5221$\pm$1.0119  & 7.1352$\pm$1.4439 & \bf{10.2148}$\pm$1.1338\\ 
\hline
$AFDS$ (Ours)   & 1.4386$\pm$0.9740  & 2.3015$\pm$1.0139  &  2.6126$\pm$1.5182 & 3.8321$\pm$1.4084\\ 
\hline
$GFDS$ (Ours)   &  1.4103$\pm$1.0369  &  2.1501$\pm$1.6841 &  2.5322$\pm$1.1527  & 2.4206$\pm$ 1.1234\\ 
\hline
$GFDS^+$ (Ours) & 1.9516$\pm$1.5658 & \bf{ 6.8731}$\pm$1.1930  &  \bf{7.8081}$\pm$1.4934  & 7.8974$\pm$1.0149 \\ 
\hline
\end{tabular}
\label{table11}
\end{table*}

\subsection{Theoretical Analysis}
We carry out theoretical analysis on $AFDS$ and $GFDS^+$. $GFDS$ is based on Owen value theory~\cite{gimenez2019owen,casajus2009shapley}, so no further analysis is carried out here. The proofs of the theorems are given in the appendix.
\newtheorem{thm}{\bf Theorem} 
\begin{thm}\label{thm1}
There are $n$ training data $\{(x_{i},y_{i})|i\in [n]\}$. Given a value function $v$ and an approximate value function $v_a$ by $AFDS$.
If $|v_a - v|<\epsilon$, then $|\phi_i(v_a)-\phi_i(v)|\le 2\epsilon$ 
\end{thm}

\begin{thm}\label{thm2}
There are $n$ training data $D=\{(x_{i},y_{i})|i\in [n]\}$. Let us now define a partition $G = \{G_1,\dots, G_N\}$ of the training data, consisting of $N$ groups. Given a value function $v:\{S|S\subseteq D\}\to \mathbb{R}$. 
Then, data-wise Shapley values explaining the prediction are given by $\phi_j=\frac{1}{n}\sum_{S\subseteq D\textbackslash \{j\}}{\tbinom{n-1}{|S|}}^{-1}(v(S\cup\{j\})-v(S))$.
Group-wise Shapley values for the i-th group $G_i$ are given by $\phi_{G_i}=\frac{1}{N}\sum_{S\subseteq G\textbackslash \{G_i\}}{\tbinom{N-1}{|S|}}^{-1}(v(S\cup\{G_i\})-v(S))$.
By $GFDS^+$, we have $\phi_j^+ = \frac{1}{|G_i|}\phi_{G_i}$, $\forall j\in G_i$.
Let $T_0$ and $T_{AB}$ be two disjoint subsets of the partition $G$ and $S_0$, $S_A$, $S_B$
be any three data-subsets where $S_0 \subseteq T_0$, $S_A\subseteq T_{AB}$, $S_B\subseteq T_{AB}$.
Assume $v(S_0\cup S_A)-v(S_0\cup S_B)=v(S_A)-v(S_B)$. 
If $|\phi_k-\phi_j|\le \epsilon$, $\forall\ k,\ j\in G_i$. Then, $|\phi_j^+-\phi_j|\le (1-\frac{1}{|G_i|})\epsilon$.
\end{thm}

%% file: 6_Experiment.tex
\section{Experiments}


\subsection{Experimental Setup}
\subsubsection{Datasets}

The MNIST dataset~\cite{lecun1998gradient} is a well-known handwritten digits dataset widely utilized in machine learning. The dataset includes 60,000 training images and 10,000 testing images of handwritten digits. Each image is labeled with a number from 0 to 9 and is 28$\times$28 pixels in size. The CIFAR-10 dataset~\cite{krizhevsky2009learning} is widely recognized in computer vision as a fundamental dataset for image classification. It comprises 60,000 images of size 32$\times$32 pixels, evenly distributed into 10 different categories.

\subsubsection{Baselines}
We have the following data Shapley values estimators: (1) Leave-one-out (LOO)~\cite{black2021leave}, (2) DataShapley~\cite{ghorbani2019data}, (3) Group testing Shapley~\cite{jia2019towards}, (4) CS-Shapley~\cite{schoch2022cs}, (5) FreeShap~\cite{wang2024helpful}. 

\subsubsection{Implementation details}
The explainer model's training overhead grows exponentially as the service model's training data size increases. Consequently, we set the training data size to 100 for the service model. Subsequently, we uniformly sample the training data of the explainer model from the same data distribution. Additionally, to reduce the computational overhead of baselines, the network architecture of the service model is a CNN network. The architecture of the explainer model is a pre-trained ResNet-50, but the last layer is replaced with a convolutional layer and a linear layer. The output shape is $|D|\times |Y|$, and the output represents the Shapley values of the training data corresponding to different labels. 
The settings of the parameters in the experiment are as follows:  The explainer model's learning rate is $\alpha = 2e-4$. The service model's learning rate is $lr=e-4$, hyperparameters $\beta=10$, and the sub-service models' learning rates are $e-3$. In the $AFDS$ setting, we set $K=10$. In the $GFDS$ and $GFDS^+$ settings, we set $N=|Y|=10$ to group training data based on labels. Baselines calculate the Shapley values in advance using the accuracy over the entire test set to predict the training data's contributions to new test samples.

\subsection{Comparative Experiments}
\begin{figure*}[t]
  \centering
  \includegraphics[width=0.9\textwidth]{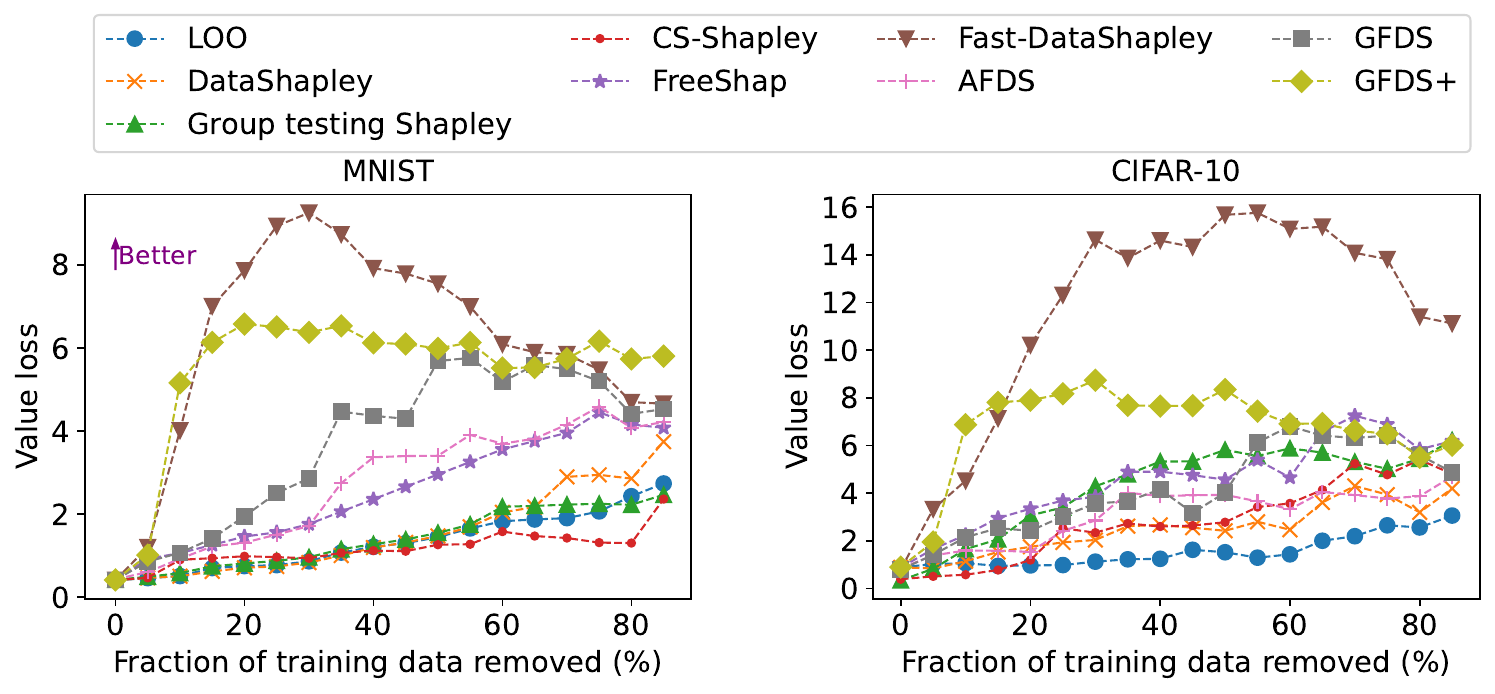}
  \caption{The comparison curve about value loss $H_\eta$ of different methods.}  
  \label{compare}
\end{figure*}

\subsubsection{Quantitative Evaluation}

Evaluating the accuracy of the Shapley value estimation requires the determination of the ground truth Shapley values, which is computationally challenging for images. So here we take data removal experiments~\cite{bhatt2020evaluating,hooker2019benchmark,wagner2019interpretable} to evaluate different training data.
We remove the top $\eta$ training data from large to small according to the predicted Shapley values of different methods. We also observe the cross entropy loss of the service model predicting the current task. 
Given a predicted sample set $T$, the metric value loss $H_\eta$ is defined as:
\begin{equation}
  H_\eta(T) = \frac{1}{|T|}\sum_{x\in T}-p(x)\log(\delta(f_{D^{'}}(x))),
\end{equation}
where $p(x)$ represents the one-hot vector of the true label. 
$D^{'}$ is the the remaining training data after removing the top $\eta$ data of $D$. A larger value of $H_\eta$ indicates the poorer performance of the model $f_{D^{'}}$, suggesting that the removed data contribute more to the predicted task.
Table \ref{table1} and \ref{table11} display the $H_\eta$ of various methods for both MNIST and CIFAR-10 datasets. We choose the drop rate $\eta$ from \{5\%,10\%,15\%,20\%\}. The results suggest that our methods outperform all of the baseline methods. 
Additionally, in Figure \ref{compare}, we present the value loss curves for both datasets. 
The curves indicate two interesting phenomena: the first is that the value loss of all methods will eventually converge to the same level due to the randomization of model performance with the increasing removal of the training data. 
For the ten-category data set, we have an equal number of training samples for each class. When predicting a specific sample, only about 10\% of the training samples play a key role. So removing the top 10\% of the training data will better reflect the performance. As the number of training samples for the same class decreases, the model will tend towards a phase of random guessing. Therefore, all methods will eventually converge to the same threshold. 
Secondly, $GFDS^+$ and $GFDS$ outperform both \oursys and $AFDS$ in some cases, which can be attributed to the introduction of the prior information by grouping the data according to true labels, making convergence easier for the explainer model under the same training conditions. Moreover, we can see that \oursys is still the best in the evaluation of the first 5\% of important samples.
We also present some visualization results, as demonstrated in Figure \ref{AFDS} and \ref{GFDS}. Notably, most training data contributing significantly to the predictions belong to the same categories as the test samples. The results are consistent with expectations.
However, we noticed that although some training samples do not belong to the same category as the test samples, they still play a positive role in the model predicting the test samples. We believe that such samples are abnormal samples for model training and it is difficult for the model to generalize performance. We can use this to explore the uninterpretability of deep models from the perspective of another training sample. At the same time, this discovery can also help in the screening of abnormal samples.

\subsubsection{Overhead Evaluation}
\begin{table*}[htbp]
\newcommand{\tabincell}[2]{\begin{tabular}{@{}#1@{}}#2\end{tabular}}
\renewcommand\arraystretch{1.1}
\centering
\caption{The explainer training overhead of different methods.}
\resizebox{\textwidth}{!}{
\begin{tabular}{ |c|c|c|c|c|c|c|c|c| } 
\hline
\bf{Methods}  &  DataShapley & Group testing Shapley & CS-Shapley& FreeShap & Fast-DataShapley (Ours)& $AFDS$ (Ours) & $GFDS$ (Ours) & $GFDS^+$ (Ours) \\
\hline
MNIST  & 15731 sec & 9609 sec & 9938 sec & 1628 sec & 2353 sec & 590 sec & 352 sec & 183 sec \\ 
\hline
CIFAR-10  & 33228 sec & 20952 sec & 21881 sec & 3173 sec & 4324 sec & 1120 sec & 396 sec& 259 sec\\ 
\hline
\end{tabular}
}
\label{table2}
\end{table*}

The experiments were conducted by 4 cores of AMD EPYC 7543 32-core processor and a single NVIDIA Tesla A30. 
Table \ref{table2} shows the training overhead of various methods. DataShapley costs the most time, while our methods \oursys, $AFDS$, $GFDS$ and $GFDS^+$ demonstrate sequentially decreasing time consumption.  
This aligns with the original design objective of our approaches. The results indicate that our methods have lower training overhead than the previous Shapley value-based method. Additionally, our methods allow for direct end-to-end prediction of new test samples without retraining.

\begin{figure*}[htbp]
  \centering
  \includegraphics[width=\textwidth]{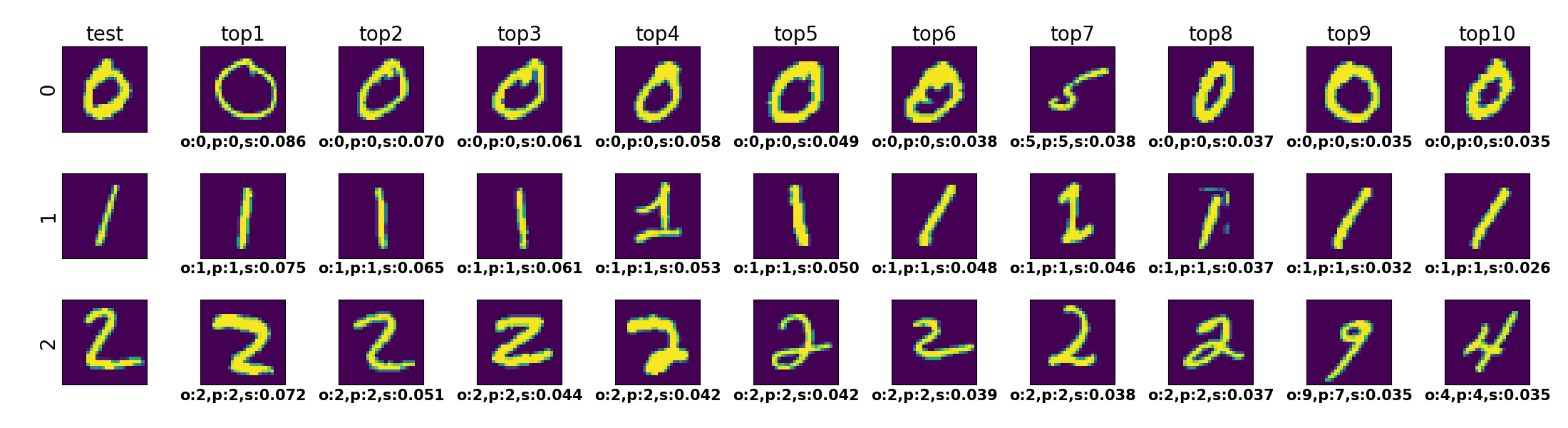}
  \caption{The visualization of $AFDS$ on MNIST shows the test samples in the first column, with the corresponding true labels on the left. The next ten columns display the top 10 predictions from the explainer model that contributed most to the test samples. $o$ represents the original label, $p$ represents the predicted label of the service model, and $s$ represents the SV.}
  \label{AFDS}
\end{figure*}

\begin{figure*}[htbp]
  \centering
  \includegraphics[width=\textwidth]{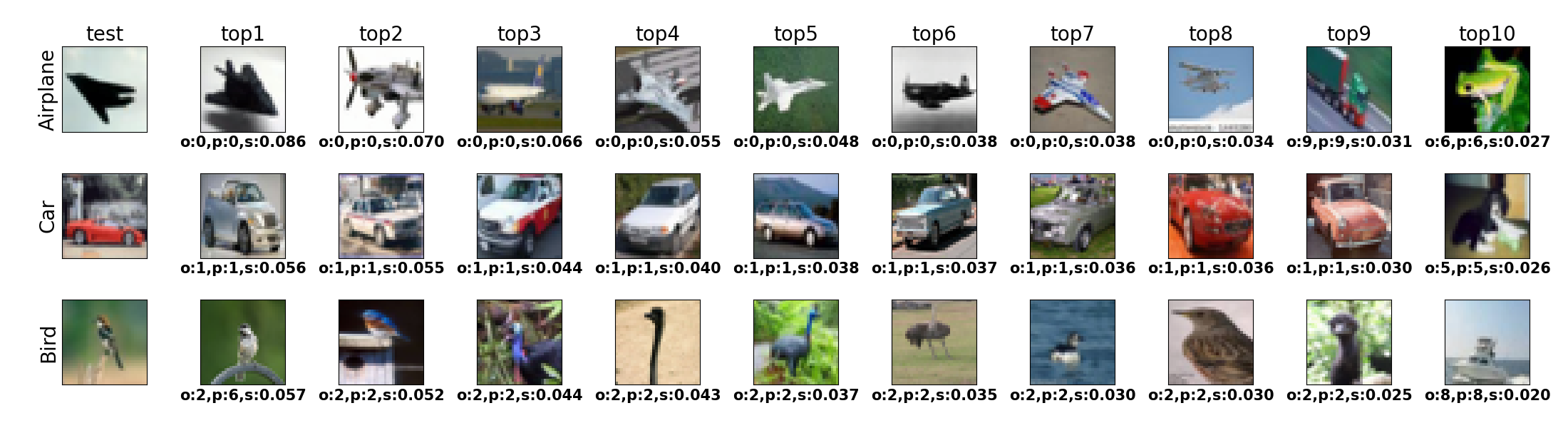}
  \caption{The visualization of $GFDS$ on CIFAR-10. }
  \label{GFDS}
\end{figure*}

%% file: 3_Related_work.tex
\section{Related Work}
There are two primary approaches for analyzing and estimating the influence of training data~\cite{hammoudeh2022training}: \textbf{RBMs} and \textbf{GBIEs}.
RBMs mean allocating training data contributions by retraining models with different subsets of the data. These methods mainly include three classes: Leave-one-out, Downsampling, and the Shapley value methods.
LOO~\cite{ling1984residuals,black2021leave,jia2021scalability,sharchilev2018finding,wojnowicz2016influence} is the simplest and oldest. It removes a single sample from the training data and trains a new model. The change in empirical risk is then used to measure the influence of this sample. 
Downsampling, proposed by Feldman and Zhang~\cite{feldman2020neural,kandpal2022deduplicating,jiang2020characterizing}, mitigates two weaknesses of LOO: computational complexity and instability. This method relies on an ensemble of submodels trained with subsets of the training data and provides a statistically consistent estimator of the leave-one-out influence.
Shapley value (SV)~\cite{jia2019efficient,ghorbani2019data,jia2019towards,covert2021improving,schoch2022cs} is derived from cooperative game theory and is widely used for fair distribution of contributions. It has four good properties: effectiveness, symmetry, redundancy, and additivity. A typical representative is DataShapley, proposed by Ghorbani and Zou~\cite{ghorbani2019data}, which distributes contributions by treating the training samples as cooperative players. DataShapley can reduce the calculation complexity by offering two SV estimators: truncated Monte Carlo Shapley and Gradient Shapley.
GBIEs estimate the influence by adjusting the gradients of training and test samples during or after the training process. Gradient-based methods mainly contain static methods~\cite{koh2017understanding,guo2020fastif,schioppa2022scaling,yeh2018representer,sui2021representer} and dynamic methods~\cite{pruthi2020estimating,yeh2022first,hara2019data,chen2021hydra, wu2022davinz, wang2024helpful}.
Static methods use only final model parameters to measure influence, such as influence functions~\cite{koh2017understanding} and representer point methods~\cite{yeh2018representer}.
Influence functions examine how the empirical risk minimizer of the model changes when one training data is perturbed. 
Representer point methods distribute contributions of each training data by decomposing predictions for specific model classes.
Dynamic methods consider parameters from the entire model training process, such as 
TracIn~\cite{pruthi2020estimating} and HyDRA~\cite{chen2021hydra}.
TracIn fixes the training data and analyzes changes in model parameters for every training iteration.
HyDRA differs from TracIn in using a Taylor series-based analysis similar to the influence functions.
However, it is worth noting that these methods do not satisfy any fairness conditions except the Shapley value. DAVINZ~\cite{wu2022davinz} and FreeShap~\cite{wang2024helpful} explore the Shapley-based techniques by using the neural tangent kernel theory~\cite{jacot2018neural}.


%% file: 8_Conclusion.tex
\section{Conclusion and Future Work}

In this work, we propose a novel framework, \oursys, to calculate the Shapley values of training data for test samples. We evaluate the training data's contributions to the prediction samples based on these Shapley values.  We also propose three methods to reduce the training overhead.
Our methods can estimate Shapley values by a single forward pass using a learned explainer model in real-time and without retraining the target model. 
Experimental results demonstrate the superior performance of our methods compared to baselines across various datasets.
While our methods proposed in this paper only cover supervised tasks, there exists potential adaption for generative tasks like AIGC. 
However, for AIGC tasks, there are some new challenges to address. Firstly, retraining large AIGC models is greatly expensive, then designing a reasonable utility function poses another significant challenge. Therefore, further exploratory efforts are required.

%% file: 9_appendix.tex
\section{Appendix}
\subsection{Proof of Theorem ~\ref{thm1}.}
\begin{proof}[Proof ]
\begin{equation}
\begin{aligned}
&|\phi_i(v_a)-\phi_i(v)|=|\frac{1}{n}\sum_{s_i\neq 1}{\tbinom{n-1}{\vec1^T s}}^{-1}(v_a(s+e_i)-
v_a(s))-\\
&\frac{1}{n}\sum_{s_i\neq 1}{\tbinom{n-1}{\vec1^T s}}^{-1}(v(s+e_i)-v(s))|\le \sum_{s_i\neq 1}1/n{\tbinom{n-1}{\vec1^T s}}^{-1}\\
&(|v_a(s+e_i) - v(s+e_i)|+|v_a(s)-v(s)|)\\
&\le\sum_{s_i\neq 1}2/n{\tbinom{n-1}{\vec1^T s}}^{-1}\epsilon=2\epsilon.
\end{aligned}
\end{equation}
\end{proof}

\subsection{Proof of Theorem ~\ref{thm2}.}

Inspired by groupShapley equivalence~\cite{jullum2021groupshapley}. We extend the properties of group-wise Shapley values from the feature level to the training data level. We first present and prove the lemma that will be useful in the proof. 
\begin{lemma}

There are $n$ training data $D=\{(x_{i},y_{i})|i\in [n]\}$. Let us now define a partition $G = \{G_1,\dots, G_N\}$ of the training data, consisting of $N$ groups. Given a value function $v:\{S|S\subseteq D\}\to \mathbb{R}$. 
Then, data-wise Shapley values explaining the prediction are given by $\phi_j=\frac{1}{n}\sum_{S\subseteq D\textbackslash \{j\}}{\tbinom{n-1}{|S|}}^{-1}(v(S\cup\{j\})-v(S))$.
Group-wise Shapley values for the i-th group $G_i$ are given by $\phi_{G_i}=\frac{1}{N}\sum_{S\subseteq G\textbackslash \{G_i\}}{\tbinom{N-1}{|S|}}^{-1}(v(S\cup\{G_i\})-v(S))$.
Let $T_0$ and $T_{AB}$ be two disjoint subsets of the partition $G$ and $S_0$, $S_A$, $S_B$
be any three data-subsets where $S_0 \subseteq T_0$, $S_A\subseteq T_{AB}$, $S_B\subseteq T_{AB}$.
Assume $v(S_0\cup S_A)-v(S_0\cup S_B)=v(S_A)-v(S_B)$. 
\\Then, $$\phi_{G_i}=\sum_{j\in G_i}\phi_j,\ i=1,\dots,N. $$
i.e. group-wise Shapley values are the same as summing data-wise Shapley values.

\end{lemma}
\begin{proof}

\begin{equation}
\begin{aligned}
\phi_j&=\frac{1}{n}\sum_{S\subseteq D\textbackslash \{j\}}{\tbinom{n-1}{|S|}}^{-1}(v(S\cup\{j\})-v(S))\\
&=\frac{1}{n}\sum_{S_1\subseteq G_i\textbackslash \{j\}}\sum_{S_0\subseteq D\textbackslash G_i }{\tbinom{n-1}{|S_0|+|S_1|}}^{-1}(v(S_0\cup S_1\cup\{j\})\\ &-v(S_0\cup S_1))
\end{aligned}
\label{9.1}
\end{equation} 
Letting $T_0=D\textbackslash G_i $, $T_{AB}=G_i$, $S_0=S_0$, $S_A=S_1\cup\{j\}$, $S_B=S_1$ in condition. Then,
\begin{equation}
\begin{aligned}
v(S_0\cup S_1\cup\{j\})-v(S_0\cup S_1) = v( S_1\cup\{j\})-v(S_1).
\end{aligned}
\label{9.2}
\end{equation} 
We next focus on the inner sum of ~(\ref{9.1}). We have:
\begin{equation}
\begin{aligned}
\sum_{S_0\subseteq D\textbackslash G_i }{\tbinom{n-1}{|S_0|+|S_1|}}^{-1} = \sum_{k=0}^{n-|G_i|}\tbinom{n-|G_i|}{k}/\tbinom{n-1}{|S_1|+k}
\end{aligned}
\label{9.3}
\end{equation} 
By the identity:
\begin{equation}
\begin{aligned}
\int_0^1 x^a(1-x)^b dx = \frac{a!b!}{(a+b+1)!},\ \forall a,b\in N_+.
\end{aligned}
\label{9.4}
\end{equation} 
We have:
\begin{equation}
\begin{aligned}
{\tbinom{n-1}{|S_1|+k}}^{-1} = n \int_0^1 x^{|S_1|+k}(1-x)^{n-1-|S_1|-k} dx.
\end{aligned}
\label{9.5}
\end{equation} 
Then, 
\begin{equation}
\begin{aligned}
\sum_{k=0}^{n-|G_i|}&\tbinom{n-|G_i|}{k}/\tbinom{n-1}{|S_1|+k}
=n\int_0^1 x^{|S_1|}(1-x)^{n-1-|S_1|}\\
&\sum_{k=0}^{n-|G_i|}\tbinom{n-|G_i|}{k}(\frac{x}{1-x})^kdx\\
&=n\int_0^1 x^{|S_1|}(1-x)^{n-1-|S_1|}(\frac{1}{1-x})^{n-|G_i|}dx\\
&=n\int_0^1 x^{|S_1|}(1-x)^{|G_i|-|S_1|-1} dx = \frac{n}{|G_i|}\tbinom{|G_i-1|}{|S_1|}^{-1}
\end{aligned}
\label{9.6}
\end{equation} 
Put equations~(\ref{9.2}), ~(\ref{9.3}),~(\ref{9.6}) into equation~(\ref{9.1}). We get:

\begin{equation}
\begin{aligned}
\phi_j&=\frac{1}{|G_i|}\sum_{S\subseteq G_i\textbackslash \{j\}}{\tbinom{|G_i|-1}{|S|}}^{-1}(v(S\cup\{j\})-v(S))
\end{aligned}
\label{9.7}
\end{equation} 
Then,
\begin{equation}
\begin{aligned}
\sum_{j\in G_i}\phi_j = v(G_i)-v(\emptyset)
\end{aligned}
\label{9.8}
\end{equation} 
Furthermore, letting $T_0=G\textbackslash G_i $, $T_{AB}=G_i$, $S_0=S$, $S_A=G_i$, $S_B=\emptyset$, then we have:
\begin{equation}
\begin{aligned}
\phi_{G_i}&=\frac{1}{N}\sum_{S\subseteq G\textbackslash \{G_i\}}{\tbinom{N-1}{|S|}}^{-1}(v(S\cup\{G_i\})-v(S)) \\
&=(v(\{G_i\})-v(\emptyset)\frac{1}{N}\sum_{S\subseteq G\textbackslash \{G_i\}}{\tbinom{N-1}{|S|}}^{-1}\\
&=v(G_i)-v(\emptyset)=\sum_{j\in G_i}\phi_j
\end{aligned}
\label{9.9}
\end{equation} 
Thus,
\begin{equation}
\begin{aligned}
\phi_j^+-\phi_j &= \frac{1}{|G_i|}\phi_{G_i}-\phi_j= \frac{1}{|G_i|}\sum_{k\in G_i}\phi_k-\phi_j\\
&=\frac{1}{|G_i|}\sum_{k\in G_i,k\neq j}(\phi_k-\phi_j)
\end{aligned}
\label{9.10}
\end{equation} 
By Cauchy-Buniakowsky-Schwarz inequality, we have, 
\begin{equation}
\begin{aligned}
&|\phi_j^+-\phi_j|=|\frac{1}{|G_i|}\sum_{k\in G_i,k\neq j}(\phi_k-\phi_j)|\\
&\le \frac{1}{|G_i|}(\sum_{k\in G_i,k\neq j}1^2)^{\frac{1}{2}}(\sum_{k\in G_i,k\neq j}(\phi_k-\phi_j)^2)^{\frac{1}{2}}\le (1-\frac{1}{|G_i|})\epsilon
\end{aligned}
\label{9.11}
\end{equation} 

\end{proof}